\documentclass[11pt]{article}
\usepackage{fullpage}\usepackage[margin=1.5in]{geometry}

\usepackage{microtype}
\usepackage{subcaption}

\usepackage{natbib}
\renewcommand{\cite}{\citep}
\newcommand{\scite}{\citeyear}

\usepackage{times}
\usepackage{soul}
\usepackage{url}
\usepackage[breaklinks=true]{hyperref}
\usepackage[utf8]{inputenc}
\usepackage{graphicx}
\usepackage{amsmath}
\usepackage{amsthm}
\usepackage{booktabs}
\usepackage{algorithm}
\usepackage{algorithmic}
\usepackage{amsfonts}
\usepackage{textcomp}
\usepackage{xcolor}
\usepackage{enumitem}
\usepackage{multirow}

\usepackage{authblk}

\DeclareMathOperator*{\argmax}{\arg\max}

\iffalse
\newcommand\eugeneie[1]{}
\newcommand\cboutilier[1]{}
\newcommand\wrui[1]{}
\newcommand{\todo}[1]{}
\else
\newcommand\eugeneie[1]{[\textcolor{blue}{EI: {#1}}]}
\newcommand\cboutilier[1]{[\textcolor{teal}{CEB: {#1}}]}
\newcommand\wrui[1]{[\textcolor{orange}{RW: {#1}}]}
\newcommand{\todo}[1]{\textcolor{red}{TODO: {#1}}}
\fi

\newcommand{\calS}{\mathcal{S}}
\newcommand{\calA}{\mathcal{A}}
\newcommand{\calI}{\mathcal{I}}
\newcommand{\calN}{\mathcal{N}}
\newcommand{\pctr}{\mathit{pCTR}}
\newcommand{\qbar}{\overline{Q}}

\newcommand{\scassm}{\textbf{SC}}
\newcommand{\rtassm}{\textbf{RTDS}}

\newcommand{\bfd}{\mathbf{d}}
\newcommand{\bfy}{\mathbf{y}}

\newcommand{\bfu}{\mathbf{u}}
\newcommand{\bfe}{\mathbf{e}}
\newcommand{\veps}{\varepsilon}

\newcommand{\SlateQ}{{\textsc{SlateQ}}}

\newtheorem{thm}{Theorem}
\newtheorem{prop}[thm]{Proposition}

\newtheorem{obs}[thm]{Observation}
\theoremstyle{definition}

\newcommand{\denselist}{\itemsep 0pt\partopsep 0pt}

\title{Reinforcement Learning for Slate-based Recommender Systems: A Tractable Decomposition and Practical Methodology\thanks{Parts of this paper appeared in \cite{slateQ:ijcai19}.}}

\author[1]{Eugene Ie$^{\dagger,\ddagger,}$}
\author[1]{Vihan Jain$^{\ddagger,}$}
\author[1]{Jing Wang$^{\ddagger,}$}
\author[2]{Sanmit Narvekar$^{\mathsection,}$}
\author[1]{Ritesh Agarwal}
\author[1]{Rui Wu}
\author[1]{Heng-Tze Cheng}
\author[3]{Morgane Lustman}
\author[3]{Vince Gatto}
\author[3]{Paul Covington}
\author[3]{Jim McFadden}
\author[1]{Tushar Chandra}
\author[1]{Craig Boutilier$^{\dagger,}$}
\affil[1]{Google Research}
\affil[2]{Department of Computer Science, University of Texas at Austin}
\affil[3]{YouTube, LLC}

\begin{document}

\maketitle

\let\oldthefootnote\thefootnote
\renewcommand{\thefootnote}{\fnsymbol{footnote}}
\footnotetext[2]{Corresponding authors: \url{{eugeneie,cboutilier}@google.com}.}
\footnotetext[3]{Authors Contributed Equally.}
\footnotetext[4]{Work done while at Google Research.}
\let\thefootnote\oldthefootnote

\begin{abstract}
\begin{small}
Most practical recommender systems focus on estimating immediate user engagement without considering the long-term effects of recommendations on user behavior. 
Reinforcement learning (RL) methods offer the potential to optimize recommendations for long-term user engagement. However, since users are often presented with slates of multiple items---which may have interacting effects on user choice---methods are required to deal with the combinatorics of the RL action space. In this work, we address the challenge of making slate-based recommendations to optimize
long-term value using RL. Our contributions are three-fold. (i) We develop \textsc{SlateQ}, a decomposition of value-based temporal-difference and Q-learning that renders RL tractable with slates. Under mild assumptions on user choice behavior, we show that the long-term value (LTV) of a slate can be decomposed into a tractable function of its component item-wise LTVs. (ii) We outline a methodology that leverages existing myopic learning-based recommenders to quickly develop a recommender that handles LTV. (iii) We demonstrate our methods in simulation, and validate the scalability of decomposed TD-learning using \SlateQ{} in live experiments on YouTube.
\end{small}
\end{abstract}

\section{Introduction}
\label{sec:intro}

Recommender systems have become ubiquitous, transforming user
interactions with products, services and content in a wide variety 
of domains. In content recommendation, recommenders generally surface
relevant and/or novel personalized
content based on learned models of \emph{user preferences}
(e.g., as in collaborative filtering \cite{breese-cf:uai98,grouplens:cacm97,srebro:nips04,salakhutdinov-mnih:nips07}) or predictive
models of \emph{user responses} to specific recommendations. Well-known
applications of recommender systems include video recommendations
on YouTube \cite{covington:recsys16}, movie recommendations on Netflix \cite{gomez2015netflix} and playlist construction on Spotify \cite{jacobson2016music}.
It is increasingly common to train deep neural networks
(DNNs) \cite{vandenoord:nips13,wang_etal:kdd15,covington:recsys16,cheng_wide_and_deep:dlrs16} to predict user responses (e.g.,
click-through rates, content engagement, ratings, likes) to
generate, score and serve candidate recommendations.

Practical recommender systems largely focus on \emph{myopic} prediction---estimating a user's \emph{immediate} response to a recommendation---without considering the long-term impact on subsequent user behavior. This can be limiting: modeling a recommendation's stochastic impact on the future affords opportunities to trade off user engagement in the near-term for longer-term benefit (e.g., by probing a user's interests, or improving satisfaction). As a result, research has increasingly turned to the sequential nature of user behavior using temporal models, such as hidden Markov models and recurrent neural networks \cite{rendle:www10,campos:umuai14,ruining_he:icdm16,sahoo:mis2012,tan_rnn_recommender:2016,ahmed_smola_et_al:wsdm17}, and long-term planning using \emph{reinforcement learning (RL)} techniques (e.g., \cite{shani:jmlr05,taghipour:recsys07,facebook_horizon:2018}). However, the application of RL has largely been confined to restricted domains due to the complexities of putting such models into practice at scale.

In this work, we focus on the use of RL to maximize long-term value (LTV) of recommendations to the user, specifically, long-term user engagement. We address two key challenges facing the deployment of RL in practical recommender systems, the first algorithmic and the second methodological.

Our first contribution focuses on
the algorithmic challenge of \emph{slate recommendation} in RL.
One challenge in many recommender systems is that, rather than a single item,
multiple items are recommended to a user simultaneously; that is, users are
presented with a \emph{slate} of recommended items. This induces a RL problem with a large combinatorial action space, which in turn imposes significant demands on exploration, generalization and action optimization. Recent approaches to RL with such combinatorial actions \cite{sunehag2015deep,metz2017discrete} make inroads into this problem, but are unable to scale to problems of the size encountered in large, real-world recommender systems, in part because of their generality. In this work, we develop a new \emph{slate decomposition} technique called \emph{\SlateQ{}} that estimates the \emph{long-term value (LTV)} of a slate of items by directly using the estimated LTV of the \emph{individual items on the slate}. This decomposition exploits certain assumptions about the specifics of user choice behavior---i.e., the process by which user preferences dictate selection and/or engagement with items on a slate---but these assumptions are minimal and, we argue below, very natural in many recommender settings. 

More concretely, we first show how the \SlateQ{} decomposition can be incorporated into \emph{temporal difference (TD)} learning algorithms, such as SARSA and Q-learning, so that LTVs can be learned at the level of individual items despite the fact that items are always presented to users in slates. This is critical for both generalization and exploration efficiency. We then turn to the optimization problem required to build slates that maximize LTV, a necessary component of policy improvement (e.g., in Q-learning) at training time and for selecting optimal slates at serving time. Despite the combinatorial (and fractional) nature of the underlying optimization problem, we show that it can be solved in polynomial-time by a two-step reduction to a linear program (LP). We also show that simple top-$k$ and greedy approximations, while having no theoretical guarantees in this formulation, work well in practice.

Our second contribution is methodological. Despite the recent successes of RL afforded by deep Q-networks (DQNs) \cite{mnih2015,silver2016}, the deployment of RL in practical recommenders is hampered by the need to construct relevant state and action features for DQN models, and to train models that serve millions-to-billions of users. In this work, we develop a methodology that allows one to exploit existing myopic recommenders to: (a) accelerate RL model development; (b) reuse existing training infrastructure to a great degree; and (c) reuse the same serving infrastructure for scoring items based on their LTV. Specifically, we show how \emph{temporal difference (TD)} learning can be built on top of existing myopic pipelines to allow the training and serving of DQNs.

Finally, we demonstrate our approach with both offline simulation experiments and online live experiments on the YouTube video recommendation system. We show that our techniques are scalable and offer significant improvements in user engagement over myopic recommendations. The live experiment also demonstrates how our methodology supports the relatively straightforward deployment of TD and RL methods that build on the learning infrastructure of extant myopic systems.

The remainder of the paper is organized as follows. 
In Section~\ref{sec:related}, we briefly discuss related work on the use of RL for recommender systems, choice modeling,  and RL with combinatorial action spaces.
We formulate the LTV slate recommendation problem as a Markov decision process (MDP) in Section~\ref{sec:setup} and briefly discuss standard value-based RL techniques, in particular, SARSA and Q-learning.

We introduce our \SlateQ{} decomposition in Section~\ref{sec:slate-decomposition}, discussing the assumptions under which the decomposition is valid, and how it supports effective TD-learning by allowing the long-term value (Q-value) of a slate to be decomposed into a function of its constituent item-level LTVs (Q-values). We pay special attention to the form of the \emph{user choice model}, i.e., the random process by which a user's preferences determine the selection of an item from a slate. The decomposition
affords \emph{item-level exploration and generalization} for TD methods like SARSA and Q-learning, thus obviating the need to construct value or Q-functions explicitly over slates. For Q-learning itself to be feasible, we must also solve the combinatorial \emph{slate optimization problem}---finding a slate with maximum LTV given the Q-values of individual items. We address this problem in Section~\ref{sec:slateopt}, showing that it can be solved effectively by first developing a fractional mixed-integer programming formulation for slate optimization, then deriving a reformulation and relaxation that allows the problem to be solved exactly as a linear program. We also describe two simple heuristics, \emph{top-$k$} and \emph{greedy} slate construction, that have no theoretical guarantees, but perform well in practice.

To evaluate these methods systematically, we introduce a recommender simulation environment, \emph{RecSim}, in Section~\ref{sec:simulator} that allows the straightforward configuration of an item collection (or vocabulary), a user (latent) state model and a user choice model. We describe specific instantiations of this environment suitable for slate recommendation, and in 
Section~\ref{sec:empiricalSim} we use these models in the empirical evaluation of 
our \SlateQ{} learning and optimization techniques.

The practical application of RL in the estimation of LTV in large-scale, practical recommender systems often requires integration of RL methods with production machine-learning training and serving infrastructure. In Section~\ref{sec:implementation}, we outline a general methodology by which RL methods like \SlateQ{} can be readily incorporated into the typical infrastructure used by many myopic recommender systems. We use this methodology to test the \SlateQ{} approach, specifically using SARSA to get one-step policy improvements, in a live experiment for recommendations on the YouTube homepage. We discuss the results of this experiment in Section~\ref{sec:empiricalLive}.

\section{Related Work}
\label{sec:related}

We briefly review select related work in recommender systems, choice modeling and combinatorial action optimization in RL.

\paragraph{Recommender Systems}


Recommender systems have typically relied on collaborative filtering (CF) techniques \cite{grouplens:cacm97,breese-cf:uai98}. These exploit user feedback on a subset of items (either explicit, e.g., ratings, or implicit, e.g., consumption) to directly estimate user preferences for unseen items. CF techniques include methods that explicitly cluster users and/or items, methods that embed users and items in a low-dimensional representation (e.g., LSA, probabilistic matrix factorization), or combinations of the two \cite{krestel2010tagLDA,moshfeghi2011lda}.

Increasingly, recommender systems have moved beyond explicit preference prediction to capture more nuanced aspects of user behavior, for instance, how they respond to specific recommendations, such as pCTR (predicted click-through rate), degree of engagement (e.g., dwell/watch/listen time), ratings, social behavior (e.g., comments, sharing) and other behaviors of interest. DNNs now play a significant role in such approaches \cite{vandenoord:nips13,wang_etal:kdd15,covington:recsys16,cheng_wide_and_deep:dlrs16} and often use CF-inspired embeddings of users and items as inputs to the DNN itself.

\paragraph{Sequence Models and RL in Recommender Systems}

Attempts to formulate recommendation as a RL problem have been relatively uncommon, though it has attracted more attention recently. Early models included a MDP model for shopping recommendation \cite{shani:jmlr05} and Q-learning for page navigation \cite{taghipour:recsys07}, but were limited to very small-scale settings (100s of items, few thousands of users). More recently, biclustering has been combined with RL algorithms for recommendation \cite{choi2018reinforcement},
while \citet{facebook_horizon:2018} describe the use of RL in several applications at Facebook. \citet{chen_etal:2018top} also explored a novel off-policy policy-gradient approach that is very scalable, and was shown to be effective in a large-scale commercial recommender system. Their approach does not explicitly compute LTV improvements (as we do by developing Q-value models), nor does it model the slate effects that arise is practical recommendations. 

\citet{zhao_slateRL:recsys18} explicitly consider RL in slate-based recommendation systems, developing an actor-critic approach for recommending a page of items and tested using simulator trained on user logs. While similar in motivation to our approach, this method differs in several important dimensions: it makes no significant structural assumptions about user choice, using a CNN to model the spatial layout of items on a page, thus not handling the action-space combinatorics w.r.t.\ generalization, exploration, or optimization (but allowing additional flexibility in capturing user behavior). Finally, the focus of their method is online training and their evaluation with offline data is limited to item reranking. 

\paragraph{Slate Recommendation and Choice Modeling}

Accounting for slates of items in recommender systems is quite common \cite{deshpande:tois04,activecf:uai03,viappiani:nips2010,le:cikm17}, and the extension introduces interesting modeling questions (e.g., involving metrics such as diversity \cite{wilhelm:cikm18}) and computational issues due to the combinatorics of slates themselves. \citet{swaminathan_etal:nips17} explored off-policy evaluation and optimization using inverse propensity scores in the context of slate interactions. \citet{mehrotra:www19} developed a hierarchical model for understanding user satisfaction in slate recommendation.

The construction of optimal recommendation slates generally depends on \emph{user choice behavior}. Models of user choice from sets of items is studied under the banner of \emph{choice modeling} in areas of econometrics, psychology, statistics, operations research and marketing and decision science \cite{luce59,louviere-et-al:statedchoice2000}. Probably the most common model of user choice is the \emph{multinomial logit (MNL)} model and its extensions (e.g., the conditional logit model, the mixed logit model, etc.)---we refer to \citet{louviere-et-al:statedchoice2000} for an overview.

For example, the conditional logit model assumes a set of user-item characteristics (e.g., feature vector) $x_{ij}$ for user $i$ and item $j$, and determines the (random) utility $u(x_{ij})$ of the item for the user. Typically, this model is linear so $u(x_{ij}) = \beta x_{ij}$, though we consider the use of DNN regressors to estimate these logits below.
The probability of the user selecting $j$ from a slate $A$ of items is
\begin{equation}
P(j|A) = \frac{e^{u(x_{ij})}}{\sum_{\ell\in A} e^{u(x_{i\ell})}} \label{eq:conditionalLogit}
\end{equation}
The choice model is justified under specific independence and extreme value assumptions \cite{mcfadden_condlogit:1974,train:discretechoicebook2009}. Various forms of such models are used to model consumer choice and user behavior in wide ranging domains, together with specific methods for model estimation, experiment design and optimization. Such models form the basis of optimization procedures in  revenue management \cite{talluri_vanryzin:ms04,rusmevichientong_robust:or2012},
product line design \cite{chen:mgmtsci2000,schon:mgmtsci2010}, assortment optimization \cite{martinez-roels:pom11,honhon:msom12} and a variety of other areas---we exploit connections with this work in Section~\ref{sec:slateopt} below.

The conditional logit model is an instance of a more general conditional choice format in which a user $i$ selects item $j\in A$ with unnormalized probability $v(x_{ij})$, for some function $v$:
\begin{equation}
P(j|A) = \frac{v(x_{ij})}{\sum_{\ell\in A} v(x_{i\ell})}. \label{eq:conditionalChoice}
\end{equation}
In the case of the conditional logit, $v(x_{ij}) = e^{\tau u(x_{ij})}$, but any arbitrary $v$ can be used.

Within the ML community, including recommender systems and learning-to-rank, other choice models are used to explain user choice behavior. For example, \emph{cascade models} \cite{joachims2002,craswell:wsdm08,kveton_cascading:icml15} have proven popular as a means of explaining user browsing behavior through (ordered) lists of recommendations, search results, etc., and is especially effective at capturing position bias. The standard cascade model assumes that a user $i$ has some affinity (e.g., perceived utility) $u_{ij_k}$ for any item $j_k$; sequentially scans a list of items $A = (j_1, j_2, \ldots, j_K)$ in order; and will select (e.g., click) an item with probability $\phi(u_{ij_k})$ for some non-decreasing function $\phi$.  If an item is selected when inspected, no items following will be inspected/selected; and if the last item is inspected but not selected, then no selection is made. Thus the probability of $j_k$ being selected is:
\begin{equation}
P(j_k | A) = \prod_{\ell < k} (1-\phi(u_{ij_\ell})) \phi(u_{ij_k}). \label{eq:cascade}
\end{equation}
Various mechanisms for model estimation, optimization and exploration have been proposed for the basic cascade model and its variations. Recently, DNN and sequence models have been developed for explaining user choice behavior in a more general, non-parametric fashion \cite{ai2018sigir,seq2slate_full:arxiv18}. As one example, \citet{jiang_slates_cvae:iclr18} proposed a slate-generation model using conditional variational autoencoders to model the distribution of slates conditioned on user response, but the scalability requires the use of a pre-trained item embedding in large domains of the type we consider. However, the CVAE model does offer considerably flexibility in capturing item interactions, position bias, and other slate effects that might impact user response behavior.

\paragraph{RL with Combinatorial Action Spaces}

Designing tractable RL approaches for \emph{combinatorial actions}---of which slates recommendations are an example---is itself quite challenging. Some recent work in recommender systems considered slate-based recommendations (see, e.g., discussion of \citet{zhao_slateRL:recsys18} above, though they do not directly address the combinatorics), though most is more general. \emph{Sequential DQN} \cite{metz2017discrete} decomposes $k$-dimensional actions into a sequence of atomic actions, inserting fictitious states between them so a standard RL method can plan a trajectory giving the optimal action configuration. While demonstrated to be useful in some circumstances, the approach trades off the exponential size of the action space with a corresponding exponential increase in the size of the state space (with fictitious states corresponding to possible sequences of sub-actions).

\citet{sunehag2015deep} proposed \emph{Slate MDPs} which considers slates of \emph{primitive actions}, using DQN to learn the value of item slates, and a greedy procedure to construct slates. In fact, they develop three DQN methods for the problem, two of which manage the combinatorics of slates by assuming the primitive actions can be executed in isolation. In our setting, this amounts to the unrealistic assumption that we could ``force'' a user to consume a specific item (rather than present them with a slate, from which no item might be consumed). Their third approach, \emph{Generic Full Slate}, makes no such assumption, but maintains an explicit $Q$-function over slates of items. This means it fails to address the exploration and generalization problems, and while the greedy optimization (action selection) method used is tractable, it comes with no guarantees of optimality.

\section{An MDP Model for Slate Recommendation}
\label{sec:setup}

In this section, we develop a \emph{Markov decision process (MDP)} model for content recommendation with \emph{slates}. We consider a setting in which a recommender system is charged with presenting a slate to a user, from which the user selects zero or more items for consumption (e.g., listening to selected music tracks, reading content, watching video content). Once items are consumed, the user can return for additional slate recommendations or terminate the session. The user's response to a consumed item may have multiple dimensions. These may include the immediate degree of engagement with the item (e.g., consumption time); ratings feedback or comments; sharing behavior; subsequent engagement with the content provider beyond the recommender system's direct control. In this work, we use \emph{degree of engagement} abstractly as the reward without loss of generality, since it can encompass a variety of metrics, or their combinations.


We focus on \emph{session optimization} to make the discussion concrete, but our decomposition applies equally well to any long-term horizon.\footnote{Dealing with very extended horizons, such as \emph{lifetime value} \cite{theocharous:ijcai15,hallak_mansour_rl_ltv:2017}, is often problematic for any RL method; but such issues are independent of the slate formulation and decomposition we propose.} Session optimization with slates can be modeled as a MDP with states $\calS$, actions $\calA$, reward function $R$ and transition kernel $P$, with discount factor $0 \leq \gamma \leq 1$. 


The states $\calS$ typically reflect \emph{user state}. This includes relatively static user features such as demographics, declared interests, and other user attributes, as well as more dynamic user features, such as user context (e.g., time of day). In particular, summaries of relevant user history and past behavior play a key role, such as past recommendations made to the user; past user responses, such as recommendations accepted or passed on, the specific items consumed, and degree of user engagement with those items. The summarization of history is often domain specific (see our discussion of methodology in Section~\ref{sec:implementation}) and can be viewed as a means of capturing certain aspects of user latent state in a partially observable MDP. The state may also reflect certain general (user-independent) environment variables. We develop our model assuming a finite state space for ease of exposition, though our experiments and our methodology admit both countably infinite and continuous state features.

The action space $\calA$ is simply the set of all possible recommendation slates. We assume a fixed catalog of recommendable items $\calI$, so actions are the subsets $A\subseteq \calI$ such that $|A| = k$, where $k$ is the slate size. We assume that each item $a\in\calI$ and each slate $A$ is recommendable at each state $s$ for ease of exposition. However, our methods apply readily when certain items cannot be recommended at particular states by specifying $\calI_s$ for each $s\in\calS$ and restricting $\calA_s$ to subsets of $\calI_s$. If additional constraints are placed on slates so that $\calA_s$ is a \emph{strict} subset of the slates defined over $\calI_s$, these can be incorporated into the slate optimization problem at both training and serving time.\footnote{We briefly describe where relevant adjustments are needed in our algorithms when we present them. We also note that our methods work equally well when the feasible set of slates $\calA_s$ is stochastic (but stationary) as in \cite{sasmdps:ijcai18}.} We do not account for positional bias or ordering effects within the slate in this work, though such effects can be incorporated into the choice model (see below).

To account for the fact that a user may select no item from a slate, we assume that every slate includes a $(k+1)$st \emph{null item}, denoted $\bot$. This is standard in most choice modeling work and makes it straightforward to specify all user behavior as if induced by a choice from the slate.

Transition probability $P(s'|s,A)$ reflects the probability that the user transitions to state $s'$ when action $A$ is taken at user state $s$. This generally reflects uncertainty in both user response and the future contextual or environmental state. One of the most critical points of uncertainty pertains the probability with which a user will consume a particular recommended item $a\in A$ from the slate. As such, choice models play a critical role in evaluating the quality of a slate as we detail in the next section.

Finally, the reward $R(s,A)$ captures the expected reward of a slate, which measures the expected degree of user engagement with items on the slate. Naturally, this expectation must account for the uncertainty in user response.

Our aim is to find optimal slate recommendation as a function of the state.
A \emph{(stationary, deterministic) policy} $\pi:\calS \rightarrow \calA$ dictates the action $\pi(s)$ to be taken (i.e., slate to recommend) at any state $s$. The \emph{value function} $V^\pi$ of a fixed policy $\pi$ is:
\begin{align}
V^\pi(s) &= R(s,\pi(s)) + \gamma \sum_{s'\in\calS} P(s'|s,\pi(s)) V^\pi(s').
	\label{eq:value}
\end{align}
The corresponding action value, or \emph{Q-function}, reflects the value of
taking an action $a$ at state $s$ and then acting according to $\pi$:
\begin{align}
Q^\pi(s,A) &= R(s,A) + \gamma \sum_{s'\in\calS} P(s'|s,A) V^\pi(s').
	\label{eq:qpivalue}
\end{align}

The \emph{optimal policy} $\pi^\ast$ maximizes expected value $V(s)$ uniformly over $\calS$, and its value---the \emph{optimal value function} $V^\ast$---is given by the fixed point of the Bellman equation:
\begin{align}
V^\ast(s) &= \max_{A\in\calA} R(s,A) + \gamma \sum_{s'\in\calS} P(s'|s,A) V^\ast(s').
	\label{eq:optvalue}
\end{align}
The optimal Q-function is defined similarly:
\begin{align}
Q^\ast(s,A) &= R(s,A) + \gamma \sum_{s'\in\calS} P(s'|s,A) V^\ast(s').
    \label{eq:qvalue}
\end{align}
The optimal policy satisfies $\pi^\ast(s)=\arg\max_{A\in\calA} Q^\ast(s,A)$.

When transition and reward models are both provided, optimal policies and value functions can be computed using a variety of methods \cite{puterman}, though generally these require approximation in large state/action problems \cite{bertsekas:ndp}. With sampled data, RL methods such as TD-learning \cite{sutton:TD},
\emph{SARSA} 
\cite{rummery_sarsa:1994,sutton_sarsa:nips96}
and \emph{Q-learning} \cite{watkins:mlj92} can be used
(see \cite{sutton:rlbook} for an overview).
Assume training data of the form $(s, A, r, s', A')$ representing observed transitions and rewards generated by some policy $\pi$. The Q-function $Q^\pi$ can be estimated using SARSA updates of the form:
\begin{align}
Q^{(t)}(s,A)
       &\leftarrow \alpha^{(t)} [r + \gamma Q^{(t-1)}(s',A')] 
		      + (1 - \alpha^{(t)}) Q^{(t - 1)}(s,A),
		      \label{eq:sarsaupdate} 
\end{align}
where $Q^{(t)}$ represents the $t$th iterative estimate of $Q^\pi$ and $\alpha$ is the learning rate.
SARSA, Eq.~(\ref{eq:sarsaupdate}), is \emph{on-policy} and estimates the value of the data-generating policy $\pi$. However, if the policy has sufficient exploration or other forms of stochasticity (as is common in large recommender systems), acting greedily w.r.t.\ $Q^\pi$, and using the data so-generated to train a new $Q$-function, will implement a policy improvement step \cite{sutton:rlbook}. With repetition---i.e., if
the updated $Q^\pi$ is used to make recommendations (with some exploration), from which new training data is generated---the process will converge to the optimal $Q$-function. Note that acting greedily w.r.t.\ $Q^\pi$ requires the ability to compute optimal slates at serving time. In what follows, we use the term SARSA to refer to the (on-policy) estimation of the Q-function $Q^\pi(s,a)$ of a \emph{fixed} policy $\pi$, i.e., the TD-prediction problem on state-action pairs.\footnote{SARSA is often used to refer to the on-policy \emph{control} method that includes making policy improvement steps. We use it simply to refer to the TD-method based on SARSA updates as in Eq.~(\ref{eq:sarsaupdate}).}

The optimal Q-function $Q^\ast$ can be estimated directly in a related fashion:
\begin{align}
Q^{(t)}(s,A)
       &\leftarrow \alpha^{(t)} [r + \max_{A'}\gamma Q^{(t - 1)}(s', A')] 
		       + (1 - \alpha^{(t)}) Q^{(t - 1)}(s, A).
		      \label{eq:qupdate}
\end{align}
where $Q^{(t)}$ represents the $t$th iterative estimate of $Q^\ast$.
Q-learning, Eq.~(\ref{eq:qupdate}), is \emph{off-policy} and directly estimates the optimal Q-function (again, assuming suitable randomness in the data-generating policy $\pi$). Unlike SARSA, Q-learning requires that one compute optimal slates $A'$ at training time, not just at serving time.

\section{\SlateQ: Slate Decomposition for RL}
\label{sec:slate-decomposition}

One key challenge in the formulation above is the combinatorial nature of the action space, consisting of all $\binom{|\calI|}{k}\cdot k!$ ordered $k$-sets over $\calI$. This poses three key difficulties for RL methods.
First, the sheer size of the action space makes sufficient \emph{exploration} impractical. It will generally be impossible to execute all slates even once at any particular state, let alone satisfy the sample complexity requirements of TD-methods. Second, \emph{generalization} of Q-values across slates is challenging without some compressed representation. While a slate could be represented as the collection of features of its constituent items, this imposes greater demands on sample complexity; we may further desire greater generalization capabilities. Third, we must solve the combinatorial optimization problem of finding a slate with maximum Q-value---this is a fundamental part of Q-learning and a necessary component in any form of policy improvement. Without significant structural assumptions or approximations, such optimization cannot meet the real-time latency requirements of production recommender systems (often on the order of tens of milliseconds).

In this section, we develop \emph{SlateQ}, a model that allows the Q-value of a slate to be \emph{decomposed into a combination of the item-wise Q-values of its constituent items}. This decomposition exposes precisely the type of structure needed to allow effective exploration, generalization and optimization. We focus on the \SlateQ{} decomposition in this section---the decomposition itself immediately resolves the exploration and generalization concerns. We defer discussion of the optimization question to Section~\ref{sec:slateopt}.

Our approach depends to some extent on the nature of the user choice model, but critically on the interaction it has with subsequent user behavior, specifically, how it influences both expected engagement (i.e., reward) and user latent state (i.e., state transition probabilities). We require
two assumptions to derive the \SlateQ{} decomposition.
\begin{itemize} 
\item \textbf{Single choice (SC):} A user consumes a \emph{single} item from each slate
(which may be the \emph{null item} $\bot$).
\item \textbf{Reward/transition dependence on selection (RTDS):} The realized
reward (user engagement) $R(s,A)$ depends (perhaps stochastically) \emph{only on the item $i\in A$ consumed by the user} (which may also be the \emph{null item} $\bot$). Similarly, the state transition $P(s'|s,A)$
depends only on the consumed $i\in A$.
\end{itemize}
Assumption \scassm\ implies that the user selection of a subset $B\subseteq A$ from slate $A$ has $P(B|s,A) > 0$ only if $|B| = 1$. While potentially limiting in some settings, in our application (see Section~\ref{sec:empiricalLive}), users can consume only one content item at a time. Returning to the slate for a second item is modeled and logged as a separate event, with the user making a selection in a new state that reflects engagement with the previously selected item. As such, \scassm\ is valid in our setting.\footnote{Domains in which the user can select multiple items without first engaging with them (i.e., without induced some change in state) would be more accurately modeled by allowing multiple selection. Our \SlateQ{} model can be extended to incorporate a simple correction term to accurately model user selection of multiple items by assuming conditional independence of item-choice probabilities given $A$.} Letting $R(s,A,i)$ denote the reward when a user in state $s$, presented with slate $A$, selects item $i\in A$, and 
$P(s'|s,A,i)$ the corresponding probability of a transition to $s'$,
the \scassm\ assumption allows us to express immediate rewards and state transitions as follows:
\begin{align}
R(s,A) &= \sum_{i\in A} P(i|s,A) R(s,A,i), \label{eq:exprew} \\
P(s'|s,A) &= \sum_{i\in A} P(i|s,A) P(s'|s,A,i). \label{eq:exptrans}
\end{align}

The \rtassm\ assumption is also realistic in many recommender systems, especially with respect to immediate reward. It is typically the case that a user's engagement with a selected item is not influenced to a great degree by the options in the slate that were not selected.
The transition assumption also holds in recommender systems where direct user interaction with
items drives user utility, overall satisfaction, new interests, etc., and hence is the primary determinant of the user's underlying latent state. Of course, in some recommender domains, unconsumed items in the slate (say, impressions of content descriptions, thumbnails, clips, etc) may themselves create, say, future curiosity, which should be reflected by changes in the user's latent state. 
But even in such cases \rtassm\ may be treated as a reasonable simplifying assumption, especially where such impressions have significantly less impact on the user than consumed items themselves.
The \rtassm\ assumption can be stated as:
\begin{align}
R(s,A,i) &= R(s,A',i) = R(s,i), \quad \forall A, A' \textrm{ containing } i, \label{eq:decomprew} \\
P(s'|s,A,i) &= P(s'|s,A',i) = P(s'|s,i) , \quad \forall A, A' \textrm{ containing } i. \label{eq:decomptrans}
\end{align}

Our decomposition of (on-policy) Q-functions for a fixed data-generating policy $\pi$ relies on an \emph{item-wise auxiliary function} $\qbar^\pi(s,i)$, which represents the LTV of a user consuming an item $i$, i.e., the LTV of $i$ conditional on it being clicked. Under \rtassm, this function is independent of the slate $A$ from which $i$ was selected. We define:
\begin{equation}
\qbar^\pi(s,i) = R(s, i)
	  + \gamma \sum_{s'\in\calS} P(s'|s,i) V^\pi(s'). \label{eq:qpibar} 
\end{equation}
Incorporating the \scassm\ assumption, we immediately have:
\begin{prop}
\label{prop:qPiUpdate}
$Q^\pi(s,A) = \sum_{i\in A} P(i|s,A) \qbar^\pi(s,i).$
\end{prop}
\begin{proof}
This holds since:
%
%
\begin{align}
Q^\pi(s,A) &= R(s,A) + \gamma \sum_{s'\in\calS} P(s'|s,A) V^\pi(s') \label{eq:qslate}\\ 
	   &=  \sum_{i\in A} P(i|s, A) R(s, i) + \gamma \sum_{i\in A} P(i|s,A) 
	           \sum_{s'\in\calS} P(s'|s,i) V^\pi(s') \label{eq:decompProof1}\\ 
	   &=  \sum_{i\in A}\! P(i|s,A) [R(s, i)
	   + \gamma \sum_{s'\in\calS} P(s'|s,i) V^\pi\!(s')]\\ 
	   &=  \sum_{i\in A} P(i|s,A) \qbar^\pi(s,i). \label{eq:qPiSlatefinal}
\end{align}
Here Eq.~(\ref{eq:decompProof1}) follows immediately from \scassm\ and \rtassm\ (see
Eqs.~(\ref{eq:exprew}, \ref{eq:exptrans}, \ref{eq:decomprew}, \ref{eq:decomptrans})) and
Eq.~(\ref{eq:qPiSlatefinal}) follows from the definition of $\qbar^\pi$ (see Eq.~\ref{eq:qpibar}).
\end{proof}

This simple result gives a \emph{complete decomposition} of slate Q-values into Q-values for individual items. Thus, the combinatorial challenges disappear if we can learn $\qbar^\pi(s,i)$ using TD methods. Notice also that the decomposition exploits the existence of a known choice function. But apart from knowing it (and using it our Q-updates that follow), we make no particular assumptions about the choice model apart from $\scassm$. We note that learning choice models from user selection data is generally quite routine. We discuss specific choice functions in the next section and how they can be exploited in optimization.

TD-learning of the function $\qbar^\pi$ can be realized using a very simple Q-update rule. Given a consumed item $i$ at $s$ with observed reward $r$, a transition to $s'$, and selection of slate $\pi(s') = A'$, we update $\qbar^\pi$ as follows:
%
\begin{equation}
\qbar^\pi(s,i) \leftarrow \alpha (r + \gamma \sum_{j\in A'}
               P(j|s',A') \qbar^\pi(s',j))  
                + (1 -\alpha) \qbar^\pi(s,i).  \label{eq:slatesarsaupdate}
\end{equation}
The soundness of this update follows immediately from Eq.~\ref{eq:qpibar}.

Our decomposed \SlateQ{} update facilitates more compact Q-value models, \emph{using items as action inputs rather than slates}. This in turn allows for greater generalization and data efficiency. Critically, while \SlateQ{} learns item-level Q-values, it can be shown to converge to the correct \emph{slate Q-values} under standard assumptions:
\begin{prop}
\label{prop:qPiconverge}
Under standard assumptions on learning rate schedules and state-action exploration
\citep{sutton:TD,dayan:mlj92,sutton:rlbook}, and
the assumptions on user choice probabilities, state transitions, and rewards
stated in the text above, \SlateQ{}---using update~(\ref{eq:slatesarsaupdate}) and definition of slate
value~(\ref{eq:qPiSlatefinal})---will converge to the true slate Q-function $Q^\pi(s,A)$ and
support greedy policy improvement of $\pi$.
\end{prop}
\begin{proof}
\emph{(Brief sketch.)} Standard proofs of convergence of TD(0), applied to the state-action Q-function $Q^\pi$ apply directly, with the exception of the introduction of the direct \emph{expectation} over user choices, i.e., $\sum_{j\in A'} P(j|s',A')$, rather than the use of sampled choices.\footnote{We note that sampled choice could also be used in the full on-policy setting, but is problematic for optimization/action maximization as we discuss below.} However, it is straightforward to show that incorporating the explicit expectation does not impact the convergence of TD(0) (see, for example, the analysis of \emph{expected SARSA} \cite{vanseijen_exp_sarsa:adprl09}). There is some additional impact of user choice on exploration policies as well---if the choice model is such that some item $j$ has choice probability $P(j|s,A) = 0$ for \emph{any slates $A$ with $\pi(s) > 0$} in some state $s$, we will not experience user selection of item $j$ at state $s$ under $\pi$ (for value prediction of $V^\pi$ this is not problematic, but it is for learning a $Q^\pi$). Thus the exploration policy must account for the choice model, either by sampling all slates at each state (which is very inefficient), or by configuring exploratory slates that ensure each item $j$ is sampled sufficiently often. For most common choice models (see discussion below), every item has non-zero probability of selection, in which case, standard action-level exploration conditions apply. 
\end{proof}

Notice that update~(\ref{eq:slatesarsaupdate}) requires the use of a known choice model. Such choice models are quite commonly learned in ML-based recommender systems, as we discuss further below in Section~\ref{sec:implementation}. The introduction of this expectation---rather than relying on sampled user choices---can be viewed as reducing variance in the estimates much like \emph{expected SARSA}, as discussed by
\citet{sutton:rlbook} and analyzed formally by \citet{vanseijen_exp_sarsa:adprl09}. Furthermore, it is straightforward to show that the standard SARSA(0) algorithm (with policy improvement steps) will converge to the optimal Q-function, subject to the considerations mentioned above, using standard techniques \cite{singh-jaak-littman-sze:ml98,vanseijen_exp_sarsa:adprl09}.

The decomposition can be applied to Q-learning of the optimal Q-function as well,
requiring only
a straightforward modification of Eq.~(\ref{eq:qpibar}) to obtain $\qbar(s,i)$, the optimal (off-policy) conditional-on-click item-wise Q-function, specifically, replacing $V^\pi(s')$ with $V^\ast(s')$
(the proof is analogous to that of Proposition~\ref{prop:qPiUpdate}):
\begin{prop}
\label{prop:qStardecomp}
$Q(s,A) = \sum_{i\in A} P(i|s,A) \qbar(s,i).$
\end{prop}

Likewise, extending the decomposed update Eq.~(\ref{eq:slatesarsaupdate}) to full Q-learning requires only that we introduce the usual maximization:
%
%
\begin{equation}
\qbar\!(s,i) \leftarrow \alpha (r + \gamma \max_{A'\in\calA} \sum_{j\in A'}
               P(j|s',A') \qbar(s',j))  
                + (1 -\alpha) \qbar(s,i). \label{eq:slatequpdate}
\end{equation}
As above, it is not hard to show that Q-learning using this update will converge, using standard techniques \cite{watkins:mlj92,vanseijen_exp_sarsa:adprl09} and with similar considerations to those discussed in the proof sketch of Proposition~\ref{prop:qPiconverge}:
\begin{prop}
\label{prop:qStarconverge}
Under standard assumptions on learning rate schedules and sufficient exploration
\cite{sutton:rlbook}, and
the assumptions on user choice probabilities, state transitions, and rewards
stated in the text above, \SlateQ{}---using update~(\ref{eq:slatequpdate}) and definition of slate
value in Proposition~\ref{prop:qStardecomp}---will converge to the optimal
slate Q-function $Q^\ast(s,A)$.
\end{prop}

The decomposition of both the policy-based and optimal Q-functions above accomplishes two of our three desiderata: it circumvents the natural combinatorics of both exploration and generalization.
But we still face the combinatorics of action maximization: the \emph{LTV slate optimization problem} is the combinatorial optimization problem of selecting the optimal slate from $\calA$, the space of all $\binom{|\calI|}{k} k!$ possible (ordered) $k$-sets over $\calI$. This is required during training with Q-learning (Eq.~(\ref{eq:qupdate})) and when engaging in policy improvement using SARSA. One also needs to solve the slate optimization problem at serving time when executing the induced greedy policy (i.e., presenting slates with maximal LTV to users given a learned Q-function).
In the next section, we show that exact optimization is tractable and also develop several heuristic approaches to tackling this problem.




\section{Slate Optimization with Q-values}
\label{sec:slateopt}

We address the LTV slate optimization in this section. We develop
an exact linear programming formulation of the problem in Section~\ref{sec:exactopt}
using (a generalization of) the conditional logit model.
In Section~\ref{sec:greedyopt}, we describe two computationally simpler
heuristics for the problem, the top-$k$ and greedy algorithms.

\subsection{Exact Optimization}
\label{sec:exactopt}

We formulate the \emph{LTV slate optimization problem} as follows:
\begin{align}
\max_{\substack{A\subseteq\calI\\ |A| = k}}
   & \sum_{i\in A} P(i|s,A) \qbar(s,i).
	\label{eq:objective}
\end{align}
Intuitively, a user makes her choice from the slate
based on the perceived properties (e.g., attractiveness, quality, topic, utility) of the 
constituent items. In the LTV slate optimization problem, we value the selection of an item from the slate based on its LTV or $\qbar$-value, rather than its immediate appeal to the user. As discussed above, we assume access to the choice model $P(i|s,A)$, since models (e.g., pCTR models) predicting user selection from a slate are commonly used in myopic recommenders. Of course, the computational solution of the slate optimization problem depends on the form of the choice model. We discuss the use of the \emph{conditional logit model (CLM)} in \SlateQ{} (and the more general format, Eq.~(\ref{eq:conditionalChoice})) in this subsection.

When using the conditional logit model (see Eq.~\ref{eq:conditionalLogit}), the LTV slate optimization problem is analogous in a formal sense to,
\emph{assortment optimization} or \emph{product line design} \cite{chen:mgmtsci2000,schon:mgmtsci2010,rusmevichientong_robust:or2012}, in which a retailer designs or stocks a
set of $k$ products whose expected revenue or profit is maximized assuming that
consumers select products based on their appeal (and not their value to the
retailer).\footnote{Naturally, there are more complex variants of assortment optimization, including the choice of price, inclusion of fixed production or inventory costs, etc. There are other conceptual differences with our model as well. While not a formal requirement, LTV of an item in our setting reflects user engagement, hence reflects some form of user satisfaction as opposed to direct value to the recommender. In addition, many assortment optimization models are designed for consumer populations, hence choice probabilities are often reflective of diversity in the population (though random selection by individual consumers is sometimes considered as well; by contrast, in the recommender setting, choice probabilities are usually dependent on
the features of individual users and typically reflect the recommender's uncertainty about a user's immediate intent or context.}

Our optimization formulation is suited to any general conditional choice model of the form Eq.~(\ref{eq:conditionalChoice}) (of which the conditional logit is an instance).\footnote{We note that since the ordering of items within a slate does not impact choice probabilities in this model, the action (or slate) space consists of the $\binom{\calI}{k}$ \emph{unordered} $k$-sets in this case.} We assume a user in state $s$ selects item $i\in A$ with unnormalized probability $v(s,i)$, for some function $v$. In the case of the conditional logit, $v(s,i) = e^{\tau u(s,i)}$. We can express the optimization Eq.~(\ref{eq:objective}) w.r.t.\ such a $v$ as a fractional mixed-integer program (MIP), with binary variables $x_i\in\{0,1\}$ for each item $i\in\calI$ indicating whether $i$ occurs in slate $A$:
\begin{align}
\textrm{max\ }
   &\sum_{i\in\calI} \frac{x_i v(s,i) \qbar(s,i)}{v(s,\bot) + \sum_j x_j v(s,j)}
	\label{eq:MIPobjective}\\
\textrm{s.t.\ }&\sum_{i\in\calI} x_i = k; \quad x_i \in \{0,1\},\; \forall i\in\calI.
\end{align}
This is a variant of a classic product-line (or assortment) optimization problem \cite{chen:mgmtsci2000,schon:mgmtsci2010}. Our problem is somewhat simpler since there are no fixed resource costs or per-item costs.

\citet{chen:mgmtsci2000} show that that the binary indicators in this MIP can be relaxed to obtain the following fractional linear program (LP):
\begin{align}
\textrm{max\ }
   &\sum_{i\in\calI} \frac{x_i v(s,i) \qbar(s,i)}{v(s,\bot) + \sum_j x_j v(s,j)}
	\label{eq:FRACobjective}\\
\textrm{s.t.\ }&\sum_{i\in\calI} x_i = k; \quad 0\leq x_i \leq 1,\; \forall i\in\calI.
\end{align}
The constraint matrix in this relaxed problem is totally unimodular, so the optimal (vertex) solution is integral and standard non-linear optimization methods can be used.  However, since it is a fractional LP, it is directly amenable to the Charnes-Cooper \scite{charnes-cooper:1962} transformation and can be recast directly as a (non-fractional) LP. To do so, we introduce an additional variable $t$ that implicitly represents the \emph{inverse} choice weight of the selected items $t=(v(s,\bot) + \sum_j x_j v(s,j))^{-1}$, and auxiliary variables $y_i$ that represent the products $x_i\cdot (v(s,\bot) + \sum_j x_j v(s,j))^{-1}$, giving the following LP:
\begin{align}
\textrm{max\ }
   &\sum_i y_i v(s,i) \qbar(s,i)
	\label{eq:LPobjective}\\
\textrm{s.t.\ }&\ t v(s,\bot) + \sum_i y_i v(s,i) = 1\\
            & t \geq 0; \quad  \sum_i y_i \leq k t.
\end{align}

The optimal solution $(\bfy^\ast, t^\ast)$ to this LP yields the optimal $x_i$ assignment in the fractional LP Eq.~(\ref{eq:FRACobjective}) via $x_i = y^\ast_i/t^\ast$, which in turn gives the optimal slate in the original fractional MIP Eq.~(\ref{eq:MIPobjective})---just add any item to the slate where $y^\ast_i > 0$. This formulation applies equally well to the MNL model, or related random utility models. 
The slate optimization problem is now immediately proven to be polynomial-time solvable.
\begin{obs}
The LTV slate optimization problem Eq.~\ref{eq:objective}, under the general conditional choice model (including the conditional logit model), is solvable in polynomial-time in the number of items $|\calI|$ (assuming a fixed slate size $k$).
\end{obs}
Thus full Q-learning with slates using the $\SlateQ{}$ decomposition imposes at most a small polynomial-time overhead relative to item-wise Q-learning despite its combinatorial nature. 
We also note that many production recommender systems limit the set of items to be ranked using a separate retrieval policy, so the set of items to consider in the LP is usually much smaller than the complete item set. We discuss this further below in Section~\ref{sec:implementation}.

\subsection{Top-$k$ and Greedy Optimization}
\label{sec:greedyopt}

While the exact maximization of slates under the conditional choice model
can be accomplished in polynomial-time using $\qbar$ and the item-score function $v$, we may wish to avoid solving an LP at serving time.
A natural heuristic for constructing a slate is to simply add the $k$ items with
the highest score. In our \emph{top-$k$ optimization} procedure, we insert items into the slate in decreasing order of the product $v(s,i) \qbar(s,i)$.\footnote{Top-$k$ slate construction is quite common in slate-based myopic recommenders. It has recently been used in LTV optimization as well \cite{chen_etal:2018top}.}
This incurs only a $O(\log(\calI))$ overhead relative to the $O(\calI)$ time required for maximization for item-wise Q-learning. 

One problem with top-$k$ optimization is the fact that, when considering the item to add to
the $L$th slot (for $1<L\leq k$), item scores are not updated to reflect the
previous $L-1$ items already added to the slate. \emph{Greedy optimization},
instead of scoring each item \emph{ab initio}, updates item scores with respect to the current partial slate.
Specifically, given $A' = \{i_{(1)}, \ldots i_{(L-1)}\}$ of size $L-1 < k$, the $L$th item added is that with maximum marginal contribution:
%
$$\argmax_{i\not\in A'} \frac{v(s,i) \qbar(s,i) + \sum_{\ell < L} v(s,i_{(\ell)}) \qbar(s,i_{(\ell)})}
                        {v(s,i) + v(s,\bot) + \sum_{\ell < L} v(s,i_{(\ell)})}.$$
%
We compare top-$k$ and greedy optimizations with the LP solution in our offline simulation experiments below.

Under the general conditional choice model (including for the conditional logit model), neither top-$k$ nor greedy optimization will find the optimal solution as following counterexample illustrates:
\begin{center}
  \begin{tabular}{|c|c|c|}
    \hline
    Item & Score ($v(s,i)$) & Q-value \\ \hline\hline
    $\bot$ & 1 & 0 \\ \hline
    $a$ & 2 & 0.8 \\ \hline
    $b_1, b_2 $ & 1 & 1 \\ \hline
  \end{tabular}
\end{center}

The null item is always on the slate. Items $b_1, b_2$ are identical  w.r.t.\ their behavior. We have $V(\{a\}) = 1.6/3$, greater than $V(\{b_i\}) = 1/2$. Both top-$k$ and greedy will place $a$ on the slate first. However, $V(\{a,b_i\}) = 2.6/4$, whereas the optimal slate $\{b_1,b_2\}$ is valued at $2/3$. So for slate size $k=2$, neither top-$k$ nor greedy find the optimal slate.

While one might hope that the greedy algorithm provides some approximation guarantee, the set function is not submodular, which prevents standard analyses (e.g., \cite{nemhauser78,feige:jacm98,buchbinder_submodular:soda14}) from being used. The following example illustrates this.
\begin{center}
  \begin{tabular}{|c|c|c|}
    \hline
    Item & Score ($v(s,i)$) & Q-value \\ \hline\hline
    $\bot$ & 1 & 10 \\ \hline
    $a$ & 1 & 10 \\ \hline
    $b$ & 2 & $\veps$ \\ \hline
  \end{tabular}
\end{center}
We have expected values of the following item sets: $V(\emptyset) = 10; V(\{a\}) = 10; V(\{b\}) = (10+\veps)/3; V(\{a,b\}) = 5 + \veps/2$. The fact that $V(\{a\}) - V(\emptyset) < V(\{a,b\}) - V(\{b\})$ demonstrates lack of submodularity (the set function is also not monotone).\footnote{It is worth observing that without our exact cardinality constraint (sets must have size $k$), the optimal set under MNL can be computed in a greedy fashion \cite{talluri_vanryzin:ms04} (the analysis also applies to the conditional logit model).}

While we have no current performance guarantees for greedy and top-$k$, it's not hard to show that top-$k$ can perform arbitrarily poorly.
\begin{obs}
The approximation ratio of top-$k$ optimization for slate construction is
unbounded.
\end{obs}
The following example demonstrates this.
\begin{center}
  \begin{tabular}{|c|c|c|}
    \hline
    Item & Score ($v(s,i)$) & Q-value \\ \hline\hline
    $\bot$ & $\veps$ & 0 \\ \hline
    $a$ & $\veps$ & 1 \\ \hline
    $b$ & 1 & $\veps$ \\ \hline
  \end{tabular}
\end{center}
Suppose we have $k=1$. Top-$k$ scores item $b$ higher than $a$, creating the slate with value
$V(\{b\}) = \veps/(1+\veps)$, while the optimal slate has value $V(\{a\}) = 1/2$.


\subsection{Algorithm Variants}
\label{sec:variants}

With a variety of slate optimization methods at our disposal, many variations of RL algorithms exist depending on the optimization method used during training and serving. Given a trained \SlateQ{} model, we can apply that model to \emph{serve} users using either top-$k$, greedy or the LP-based optimal method to generate recommended slates. Below we use the designations TS, GS, or OS to denote these serving protocols, respectively. These designations apply equally to (off-policy) Q-learned models, (on-policy) SARSA models, and even (non-RL) \emph{myopic} models.\footnote{A myopic model is equivalent to a Q-learned model with $\gamma = 0$.}

During Q-learning, slate optimization is also required at \emph{training time} to compute the maximum successor state Q-value (Eq.~\ref{eq:slatequpdate}). This can also use either of the three optimization methods, which we designate by TT, GT, and OT
for top-$k$, greedy and optimal LP training, respectively. This designation is not applicable when training a myopic model or SARSA (since SARSA is trained only on-policy). This gives us the following
collection of algorithms. For Q-learning, we have:

\begin{center}
\begin{tabular}{|c|c||c|c|c|}
\hline
\multicolumn{2}{|c||}{} & \multicolumn{3}{c|}{Serving}\\
\cline{3-5}
\multicolumn{2}{|c||}{} & Top-$k$ & Greedy & LP (Opt) \\
\hline\hline
\multirow{3}{*}{Training} & Top-$k$ & QL-TT-TS & QL-TT-GS & QL-TT-OS \\ \cline{2-5}
    & Greedy & QL-GT-TS & QL-GT-GS & QL-GT-OS  \\ \cline{2-5}
    & LP (Opt) & QL-OT-TS & QL-OT-GS & QL-OT-OS  \\ \hline
\end{tabular}
\end{center}

For SARSA and Myopic recommenders, we have:

\begin{center}
  \begin{tabular}{|c||c|c|}
    \hline
    Serving & SARSA & Myopic \\ \hline\hline
    Top-$k$ & SARSA-TS & MYOP-TS \\ \hline
    Greedy & SARSA-GS & MYOP-GS \\ \hline
    LP (Opt) & SARSA-OS & MYOP-OS \\ \hline
  \end{tabular}
\end{center}

In our experiments below we also consider two other baselines: Random, which recommends random slates from the feasible set; and \emph{full-slate Q-learning (FSQ)}, which is a standard, non-decomposed Q-learning method that treats each slate \emph{atomically (i.e., holistically)} as a single action. The latter is a useful baseline to test whether the \SlateQ{} decomposition provides leverage for generalization and exploration.

\subsection{Approaches for Other Choice Models}
\label{sec:cascadeOpt}

The \SlateQ{} decomposition works with any choice model that satisfies the assumptions \scassm\ and \rtassm, though the form of the slate optimization problem depends crucially on the choice model. To illustrate, we consider the \emph{cascade choice model} outlined in Section~\ref{sec:related} (see, e.g., Eq.~(\ref{eq:cascade})). Notice that the cascade model, unlike the general conditional choice model, has position-dependent effects (i.e., reordering of items in the slate changes selection probabilities and the expected LTV of the slate). However, it is not hard to show that the cascade model exhibits a form of ``ordered submodularity'' if we assume that the LTV or conditional Q-value of not selecting from the slate is no greater that the Q-value of selecting any item on the slate, i.e., if $Q(s,i) \geq Q(s, \bot)$ for all $i \in \calI$.\footnote{The statements to follow hold under the weaker condition that, for all states $s$, there are at least $k$ items $i^s_1, \ldots i^s_k$ such that $Q(s, i^s_j) \geq Q(s,\bot), \forall j \leq k$ (where $k$ is the slate size).} Specificially, the value of the marginal increase in value induced by adding item $i_{\ell+1}$ to the (ordered) partial slate $(i_1, i_2, \ldots i_{\ell})$ is no greater than the increase in value of adding $i_{\ell + 1}$ to a prefix of that slate $(i_1, i_2, \ldots i_j)$ for any $j < \ell$. Thus top-$k$ optimization can be used to support training and serving of the \SlateQ{} approach under the cascade model.\footnote{It is also not hard to show that top-$k$ is not optimal for the cascade model.}

While the general conditional choice model is order-independent, in practice, it may be the case that users incorporate some elements of a cascade-like model into the conditional choice model. For example, users may devote a \emph{random} amount of time $t$ or effort to inspect a slate of $k$ recommended items, compare the top $k' \leq k$ items, where $k' = F(t)$ is some function of the available time, and select (perhaps noisily) the most preferred item from among those inspected. This model would be a reasonable approximation of user behavior in the case of recommenders that involve scrolling interfaces for example. In such a case, we end up with a distribution over slate sizes. A natural heuristic for the conditional choice model would be, once the $k$-slate is selected, to \emph{order} the items on slate based their top-$k$ or greedy scores to increase the odds that the random slate actually observed by the user contains items that induce highest expected long-term engagement.

\section{User Simulation Environment}
\label{sec:simulator}

We discuss experiments with the various \SlateQ{} algorithms in Section~\ref{sec:empiricalSim}, using a simulation environment that, though simplified and stylized, captures several essential elements of a typical recommender system that drive a need for the long/short-term tradeoffs captured by RL methods. In this section, we describe the simulation environment and models used to test \SlateQ{} in detail. We describe the environment setup in a fairly general way, as well as the specific instantiations used in our experiments, since the simulation environment may be of broader interest.


\subsection{Document and Topic Model}

We assume a set of \emph{documents} $D$ representing the content available for recommendation. We also assume a set of \emph{topics} (or user interests) $T$ that capture fundamental characteristics of interest to users; we assume topics are indexed $1, 2, \ldots |T|$. Each document $d\in D$ has an associated \emph{topic vector} $\bfd \in [0,1]^{|T|}$, where $d_j$ is the degree to which $d$ reflects topic $j$. 

In our experiments, for simplicity, each document $d$ has only a single topic $T(d)$, so $\bfd = \bfe_i$ for some $i\leq |T|$ (i.e., we have a one-hot encoding of the document topic).
Documents are drawn from content distribution $P_D$ over topic vectors, which in our one-hot topic experiments is simply a distribution over individual topics.

Each document $d\in D$ also has a length $\ell(d)$ (e.g., length of a video, music track or news article). This is sometime used as one factor in assessing potential user engagement. While the model supports documents of different lengths, in our experiments, we assume each document $d$ has the same constant length $\ell$.

Documents also have an \emph{inherent quality} $L_d$, representing the topic-independent attractiveness to the average user. Quality varies randomly across documents, with document $d$'s quality distributed according to  $\calN(\mu_{T(d)},\sigma^2)$, where $\mu_t$ is a \emph{topic-specific} mean quality for any $t\in T$. For simplicity, we assume a fixed variance across all topics. In general, quality can be estimated over time from user responses as we discuss below; but in our experiments, we assume $L_d$ is observable to the recommender system (but not to the user \emph{a priori}, see below). Quality may also be user-dependent, though we do not consider that here, since the focus of our stylized experiments is on the ability of our RL methods to learn average quality at the topic level. Both the topic and quality of a consumed document impact long-term user behavior (see Section~\ref{subsec:user-dynamics} below).

In our experiments, we use $T=20$ topics, while the precise number of documents $|D|$ is immaterial as we will see. Of these, $14$ topics are low quality, with their mean quality evenly distributed across the interval $\mu_t \in [-3,0]$. The remaining $6$ topics are high quality, with their mean quality evenly distributed across the interval $\mu_t \in [0,3]$.

\subsection{User Interest and Satisfaction Models}

Users $u\in U$ have various degrees of interests in topics, ranging from $-1$ (completely uninterested) to $1$ (fully interested), with each user $u$ associated with an \emph{interest vector} $\bfu\in [-1,1]^{|T|}$.
User $u$'s interest in document $d$ is given by the dot product $I(u,d) = \bfu\bfd$. We assume some prior distribution $P_U$ over user interest vectors, but user $u$'s interest vector is dynamic, i.e., influenced by their document consumption (see below).
To focus on how our RL methods learn to influence user interests and the quality of documents consumed, we treat a user's interest vector $\bfu$ as \emph{fully observable} to the recommender system. In general, user interests are latent, and a partially observable/belief state model is more appropriate.

A user's \emph{satisfaction} $S(u,d)$ with a consumed document $d$ is a function $f(I(u,d),L_d)$ of user $u$'s interest and document $d$'s quality. While the form of $f$ may be quite complex in general, we assume a simple convex combination $S(u,d) = (1-\alpha) I(u,d) + \alpha L_d$. Satisfaction influences user dynamics as we discuss below.

In our experiments, a new user $u$'s prior interest $\bfu$ is sampled uniformly from $\bfu\in [-1,1]^{|T|}$; specifically, there is no prior correlation across topics.
We use an extreme value of $\alpha = 1.0$ so that a user's satisfaction with a consumed document is fully dictated by document quality. This leaves user interest only to drive the selection of the document from the slate which we describe next.

\subsection{User Choice Model}

When presented with a slate of $k$ documents, a \emph{user choice} model impacts which document (if any) from the slate is consumed by the user. We assume that a user \emph{can observe} any recommended document's topic prior to selection, but \emph{cannot observe} its quality before consumption. However, the user will observe the true document quality \emph{after} consuming it. While somewhat stylized, this treatment of topic and quality observability (from the user's) perspective is reasonably well-aligned with the situation in many recommendation domains.

The general simulation environment allows arbitrary choice functions to be defined as a function of user's state (interest vector, satisfaction) and the features of the document (topic vector, quality) in the slate.
In our experiments, we use the general conditional choice model (Eq.~(\ref{eq:conditionalChoice})) as the main model for our RL methods. User $u$'s interest in document $d$, $I(u,d) = \bfu\bfd$, defines the document's relative appeal to the user and serves as the basis of the choice function.
For slates of size $k$, the null document $\bot$ is always added as a $(k+1)$st element, which (for simplicity) has a fixed utility across all users.

We also use a second choice model in our experiments, an \emph{exponential cascade model} \cite{joachims2002}, that accounts for document position on a slate. This choice model assumes ``attention'' is given to one document at a time, with exponentially decreasing attention given to documents as a user moves down the slate.
The probability that the document in position $j$ is inspected is $\beta_0\beta^j$, where $\beta_0$ is a \emph{base inspection probability} and $\beta$ is the \emph{inspection decay}. If a document is given attention, then it is selected with a \emph{base choice probability} $P(u,d)$; if the document in position $j$ is not examined or selected/consumed, then the user proceeds to the $(j+1)$st document.
The probability that the document in position $j$ is consumed is:
$$P(j,A) = \beta_0 \beta^{j}P(u,d).$$
While we don't optimize for this model, we do run experiments in which the recommender learns a policy that assumed the general conditional choice model, but users behave according to the cascade model. In this case, the base choice probability $P(u,d)$ for a document in the cascade model is set to be its normalized probability in the conditional choice model.
While the cascade model allows for the possibility of no click, even without the fictitious null document $\bot$, we keep the null document to allow the probabilities to remain calibrated relative to the conditional model. In our experiments, we use $\beta_0 = 1.0$ and $\beta = 0.65$.

\subsection{User Dynamics}
\label{subsec:user-dynamics}

To allow for a non-myopic recommendation algorithm---in our case, RL methods---to impact overall user engagement positively, we adopt a simple, but natural model of \emph{session termination}. We assume each user $u$ has an initial
\emph{budget} $B_u$ of time to engage with content during an extended session. This budget is not observable to the recommender system, and is randomly realized at session initiation using some prior $P(B)$.\footnote{Naturally, other models that do not use terminating sessions are possible, and could emphasize
amount of engagement per period.}
Each document $d$ consumed reduces user $u$'s budget by the fixed document length $\ell$. But after consumption, the quality of the document (partially) replenishes the used budget where the budget decreases by the fixed document length $\ell$ less a \emph{bonus} $b < \ell$ that increases with the document's appeal $S(u,d)$. In effect, more satisfying documents decrease the time remaining in a session at a lower rate. In particular, for any fixed topic, documents with higher quality have a higher positive impact on cumulative engagement (reduce budget less quickly) than lower quality documents. A session ends once $B_u$ reaches $0$. Since sessions terminate with probability 1, discounting is unnecessary.

In our experiments, each user's initial budget is $B_u = 200$ units of time; each consumed document uses $\ell=4$ units; and if a slate is recommended, but no document is clicked, $0.5$ units are consumed. We set bonus $b = \frac{0.9}{3.4}\cdot \ell \cdot S(u,d)$.

The second aspect of user dynamics allows user interests to evolve as a function of the documents consumed.
When user $u$ consumes document $d$, her interest in topic $T(d)$ is nudged stochastically, biased slightly towards increasing her interest, but allows some chance of decreasing her interest. Thus, a recommender faces a short-term/long-term tradeoff between nudging a user's interests toward topics that tend to have higher quality at the expense of short-term consumption of user budget.

We use the following stylized model to set the magnitude of the adjustment---how much the interest in topic $T(d)$ changes---and its polarity---whether the user's interest in topic $T(d)$ increases or decreases. 
Let $t=T(d)$ be the topic of the consumed document $d$ and $I_t$ be user $u$'s interest in topic $t$ prior to consumption of document $d$. The (absolute) change in user $u$'s interest is $\Delta_t(I_t) = (-y |I_t| + y)\cdot -I_t$, where $y\in [0,1]$ denotes the fraction of the distance between the current interest level and the maximum level $(1, -1)$ that the update move user $u$'s interest. This ensures that more entrenched interests change less than neutral interests.

In our experiments we set $y = 0.3$. A positive change in interest, $I_t\leftarrow I_t + \Delta_t(I_t)$, occurs with probability $[I(u,d)+1]/2$, and a negative change, $I_t\leftarrow I_t - \Delta_t(I_t)$, with probability $[1-I(u,d)]/2$. Thus positive (resp., negative) interests are more likely to be reinforced, i.e., become more positive (resp., negative), with the odds of such reinforcement increasing with the degree of entrenchment.

\subsection{Recommender System Dynamics}

At each stage of interaction with a user, $m$ \emph{candidate} documents are drawn from $P_D$, from which a slate of size $k$ must be selected for recommendation. This reflects the common situation in many large-scale commercial recommenders in which a variety of mechanisms are used to sub-select a small set of candidates from a massive corpus, which are in turn scored using more refined (and computationally expensive) predictive models of user engagement.

In our simulation experiments, we use $m=10$ and $k=3$. This small set of candidate documents and the small slate size is used to allow explicitly enumeration of all slates, which allows us to compare \SlateQ{} to RL methods like Q-learning that do not decompose the Q-function. In our live experiments with the YouTube platform (see Section~\ref{sec:empiricalLive}), slates are of variable size and the number of candidates is on the order of $O(1000)$.

\section{Empirical Evaluation: Simulations}
\label{sec:empiricalSim}

We now describe several sets of results designed to assess the impact of the \SlateQ{} decomposition. Our simulation environment is implemented in a general fashion, supporting many of the general models and behaviors described in the previous sections. Our RL algorithms, both those using \SlateQ{} and FSQ, are implemented using Dopamine \cite{castro18dopamine}. We use a standard two-tower architecture with stacked fully connected layers to represent user state and document. Updates to the Q-models are done online by batching experiences from user simulations. Each training-serving strategy is evaluated over 5000 simulated users for statistical significance. All results are within a 95\% confidence interval.

\subsection{Myopic vs. Non-myopic Recommendations}

We first test the quality of \emph{(non-myopic) LTV policies} learned using \SlateQ{} to optimize engagement ($\gamma=1)$, using a selection of the \SlateQ{} algorithms (SARSA vs.\ Q-learning, different slate optimizations for training/serving). We compare these to \emph{myopic scoring (MYOP)} ($\gamma=0$), which optimizes only for immediate reward, as well as a Random policy. The goal of these comparisons is to identify whether optimizing for long-term engagement using RL (either Q-learning or 1-step policy improvement via SARSA) provides benefit over myopic recommendations.

The following table compares several key metrics of the final trained algorithms (all methods use 300K training steps):
\begin{center}
    \begin{tabular}{|r||c|c|}
    \hline
        Strategy    &   Avg.\ Return (\%)   &   Avg.\ Quality (\%)  \\
    \hline\hline
        Random      &   159.2               &   -0.5929             \\
    \hline
        MYOP-TS     &   166.3 {\footnotesize (4.46\%)}      &   -0.5428 {\footnotesize (8.45\%)}    \\
        MYOP-GS     &   166.3 {\footnotesize (4.46\%)}      &   -0.5475 {\footnotesize (7.66\%)}    \\
    \hline
        SARSA-TS    &   168.4 {\footnotesize (5.78\%)}      &   -0.4908 {\footnotesize (17.22\%)}   \\
        SARSA-GS    &   172.1 {\footnotesize (8.10\%)}      &   -0.3876 {\footnotesize (34.63\%)}   \\
    \hline
        QL-TT-TS    &   168.4 {\footnotesize (5.78\%)}      &   -0.4931 {\footnotesize (16.83\%)}   \\
        QL-GT-GS    &   172.9 {\footnotesize (8.61\%)}      &   -0.3772 {\footnotesize (36.38\%)}   \\
    \hline
        QL-OT-TS    &   169.0 {\footnotesize (6.16\%)}      &   -0.4905 {\footnotesize (17.27\%)}   \\
        QL-OT-GS    &   173.8 {\footnotesize (9.17\%)}      &   -0.3408 {\footnotesize (42.52\%)}   \\
        QL-OT-OS    &   174.6 {\footnotesize (9.67\%)}      &   -0.3056 {\footnotesize (48.46\%)}   \\
    \hline
    \end{tabular}
\end{center}

The LTV methods (SARSA and Q-learning) using \SlateQ{} offer overall improvements in average return per user session. The magnitude of these improvements only tells part of the story: we also show percentage improvements relative to Random are shown in parentheses---Random gives a sense of the baseline level of cumulative reward that can be achieved without any user modeling at all. For instance, relative to the random baseline, QL-OT-GS provides a  provides a 105.6\% greater improvement than MYOP.
The LTV methods all learn to recommend documents of much higher quality than MYOP,
which has a positive impact on overall session length, which explains the improved return per user. 


We also see that LP-based slate optimization during training (OT) provides improvements over top-$k$ and greedy optimization (TT, GT) in Q-learning when comparing similar
serving regimes (e.g., QL-OT-GS vs.\ QL-GT-GS , and QL-OT-TS vs.\ QL-TT-TS).
Optimal serving (OS) also shows consistent improvement over top-$k$ and greedy serving---and greedy serving (GS) improves significantly over top-$k$ serving (TS)---when compared under the same training regime.
However, the combination of optimal training and top-$k$ or greedy serving performs well, and is especially useful when serving latency constraints are tight, since optimal training is generally done offline.

Finally, optimizing using Q-learning gives better results than on-policy SARSA (i.e., one-step improvement) under comparable training and serving regimes. But SARSA itself has significantly higher returns than MYOP, demonstrating the value of on-policy RL for recommender systems.
Indeed, repeatedly serving-then-training (with some exploration) using SARSA would implement a natural, continual policy improvement.
These results demonstrate, in this simple synthetic recommender system environment, that using RL to plan long-term interactions can provide significant value in terms of overall engagement.

\subsection{\SlateQ{} vs. Holistic Optimization}

Next we compare the quality of policies learned using the \SlateQ{} decomposition to FSQ, the non-decomposed Q-learning method that treats each slate atomically as a single action. We set $|T|=20$, $m=10$, and $k=3$ so that we can enumerate all 
$\binom{10}{3}$ slates for FSQ maximization. 
Note that the Q-function for FSQ requires representation of all 
$\binom{20}{3} = 1140$ 
slates as actions, which can impede both exploration and generalization. For \SlateQ{} we test only SARSA-TS (since this is the method tested in our live experiment below). The following table shows our results:

\begin{center}
  \begin{tabular}{|r||c|c|}
    \hline
    & Avg.\ Return (\%) & Avg.\ Quality (\%) \\ \hline\hline
    Random & 160.6 & -0.6097 \\ \hline
    FSQ & 164.2  {\footnotesize (2.24\%)} & -0.5072  {\footnotesize (16.81\%)} \\ \hline
    SARSA-TS & 170.7  {\footnotesize (6.29\%)} & -0.5340  {\footnotesize (12.41\%)} \\ \hline
  \end{tabular}
\end{center}

While FSQ, which is an off-policy Q-learning method, is guaranteed to converge to the optimal slate policy in theory with sufficient exploration, we see that, even using an \emph{on-policy method} like SARSA with a single step of policy improvement, \SlateQ{} methods perform significantly better than FSQ, offering a 180\% greater improvement over Random than FSQ. This is the case despite \SlateQ{} using no additional training-serving iterations to continue policy improvement.
This is due to the fact that FSQ must learn Q-values for 1140 distinct slates, making it difficult to explore and generalize. FSQ also takes roughly 6X the training
time of \SlateQ{} over the same number of events. These results demonstrate the considerable value of the \SlateQ{} decomposition.

Improved representations could help FSQ generalize somewhat better, but the approach is inherently unscalable, while \SlateQ{} suffers from no such limitations (see live experiment below). 
Interestingly, FSQ does converge quickly to a policy that offers recommendations of greater average quality than \SlateQ{}, but fails to make an appropriate tradeoff with user interest.


\subsection{Robustness to User Choice}

Finally, we test the robustness of \SlateQ{} to changes in the underlying user choice model. Instead of the assumed choice model defined above, users select items from the recommended slate using a simple (exponential) \emph{cascade model}, where items on the slate are inspected from top-to-bottom with a position-specific probability, and consumed with probability proportional to $I(u,d)$ if inspected. If not consumed, the next item is inspected. Though users act in this fashion, \SlateQ{} is trained using the original conditional choice model and the same decomposition is also used to optimize slates at serving time.

The following table shows results:
\begin{center}
    \begin{tabular}{|r||c|c|}
    \hline
        Strategy    &   Avg.\ Return (\%)   &   Avg.\ Quality (\%)  \\
    \hline\hline
        Random      &   159.9               &   -0.5976             \\
    \hline
        MYOP-TS     &   163.6 {\footnotesize (2.31\%)}      &   -0.5100 {\footnotesize (14.66\%)}   \\
    \hline
        SARSA-TS    &   166.8 {\footnotesize (4.32\%)}      &   -0.4171 {\footnotesize (30.20\%)}   \\
    \hline
        QL-TT-TS    &   166.5 {\footnotesize (4.13\%)}      &   -0.4227 {\footnotesize (29.27\%)}   \\
    \hline
        QL-OT-TS    &   167.5 {\footnotesize (4.75\%)}      &   -0.3985 {\footnotesize (33.32\%)}   \\
        QL-OT-OS    &   167.6 {\footnotesize (4.82\%)}      &   -0.3903 {\footnotesize (34.69\%)}   \\
    \hline
    \end{tabular}
\end{center}

\SlateQ{} continues to outperform MYOP, even when the choice model does not accurately reflect the true environment, demonstrating its relative robustness.
\SlateQ{} can be used with other choice models. For example, \SlateQ{} can be trained by assuming the cascade model, with only the optimization formulation requiring adaptation (see our discussion in Section~\ref{sec:cascadeOpt}). But since any choice model will generally be an \emph{approximation} of true user behavior, this form of robustness is critical.

Notice that QL-TT and SARSA have inverted relative performance compared to the experiments
above. This is due to the fact that Q-learning exploits the (incorrect) choice model to optimize during training, while SARSA, being on-policy, only uses the choice model to compute expectations at serving time. This suggests that an on-policy control method like SARSA (with continual policy improvement) may be more robust than Q-learning in some settings.

\section{A Practical Methodology}
\label{sec:implementation}

The deployment of a recommender system using RL or TD methods to optimize for long-term user engagement presents a number of challenges in practice. In this section, we identify several of these and suggest practical techniques to resolve them, including ways in which to exploit an existing \emph{myopic}, item-level recommender to facilitate the deployment of a non-myopic system.

Many (myopic) item-level recommender systems \cite{liu_etal:letor2009,covington:recsys16} have the following components:
\begin{enumerate}[label=(\roman*)]
    \item \emph{Logging} of impressions and user feedback;
    \item \emph{Training} of some regression model (e.g., DNN) to predict user responses for user-item pairs, which are then aggregated by some scoring function;
    \item \emph{Serving} of recommendations, ranking items by score (e.g., returning the top $k$ items for recommendation).
\end{enumerate}
Such a system can be exploited to quickly develop a non-myopic recommender system based on Q-values, representing some measure of long-term engagement, by addressing several key challenges.

\subsection{State Space Construction}

A critical part of any RL modeling is the design of the state space, that is, the development of a set of features that adequately capture a user's past history to allow prediction of long-term value (e.g., engagement) in response to a recommendation.
For the underlying process to be a MDP, the feature set should be (at least approximately) predictive of immediate user response (e.g., immediate engagement, hence \emph{reward}) and self-predictive (i.e., summarizes user history in a way that renders the implied dynamics Markovian).

The features of an extant myopic recommender system typically satisfy both of these requirements, meaning that an RL or TD model can be built using the same logged data (organized into trajectories) and the same featurization.
The engineering, experimentation and experience that goes into developing state-of-the-art recommender systems means that they generally capture (almost) all aspects of history required to predict immediate user responses (e.g., pCTR, listening time, other engagement metrics); i.e., they form a sufficient statistic. 
In addition, the core input features (e.g., static user properties, summary statistics of past behavior and responses) are often self-predictive (i.e., no further history could significantly improve next state prediction). This fact can often be verified by inspection and semantic interpretation of the (input) features.
Thus, using the existing state definition provides a natural, practical way to construct TD or RL models. We provide experimental evidence below to support this assertion in Section~\ref{sec:empiricalLive}.

\subsection{Generalization across Users}

In the MDP model of a recommender system, each user should be viewed as a separate environment or separate MDP. However, it is critical to allow for generalization across users, since few if any users generate enough experience to allow reasonable recommendations otherwise. Of course, such generalization is a hallmark of almost any recommender system. In our case, we must generalize the (implicit) MDP dynamics across users. The state representation afforded by an extant myopic recommender system is already intended to do just this, so by learning a Q-function that depends on the same user features and the myopic system, we obtain the same form of generalization.

\subsection{User Response Modeling}

As noted in Sections~\ref{sec:slate-decomposition} and \ref{sec:slateopt}, \SlateQ{} takes advantage of some measure of immediate item appeal or utility (conditioned on a specific user or state) to determine user choice behavior. In practice, since myopic recommender systems often predict these immediate responses, for example, using pCTR models, we can use these models directly to assess the immediate appeal $v(s,i)$ required by our \SlateQ{} choice model. For instance, we can use a myopic model's pCTR predictions directly as a (unnormalized) choice probabilities for items in a slate, or we can use the logits of such a model in the conditional logit choice model.
Furthermore, by using the same state features (see above), it is straightforward to build a multi-task model \cite{zhang2017asurvey} that incorporates our long-term engagement prediction with other user response predictions.

\subsection{Logging, Training and Serving Infrastructure}

The training of long-term values $Q^\pi(s,a)$ requires logging of user
data, and live serving of recommendations based on these LTV scores. The model architecture we detail below
exploits the same logging, (supervised) training and serving infrastructure
as used by the myopic recommender system.


Fig.~\ref{fig:system_overview} illustrates the structure of our LTV-based recommender system---here we focus on SARSA rather than Q-learning, since our long-term experiment in Section~\ref{sec:empiricalLive} uses SARSA.  In myopic recommender systems, the regression model predicts immediate user response (e.g., clicks, engagement), while in our non-myopic recommender system, label generation provides LTV labels, allowing the regressor to model $\qbar^\pi(s,a)$.

\begin{figure}[t]
\centering
  \includegraphics[width=0.65\linewidth]{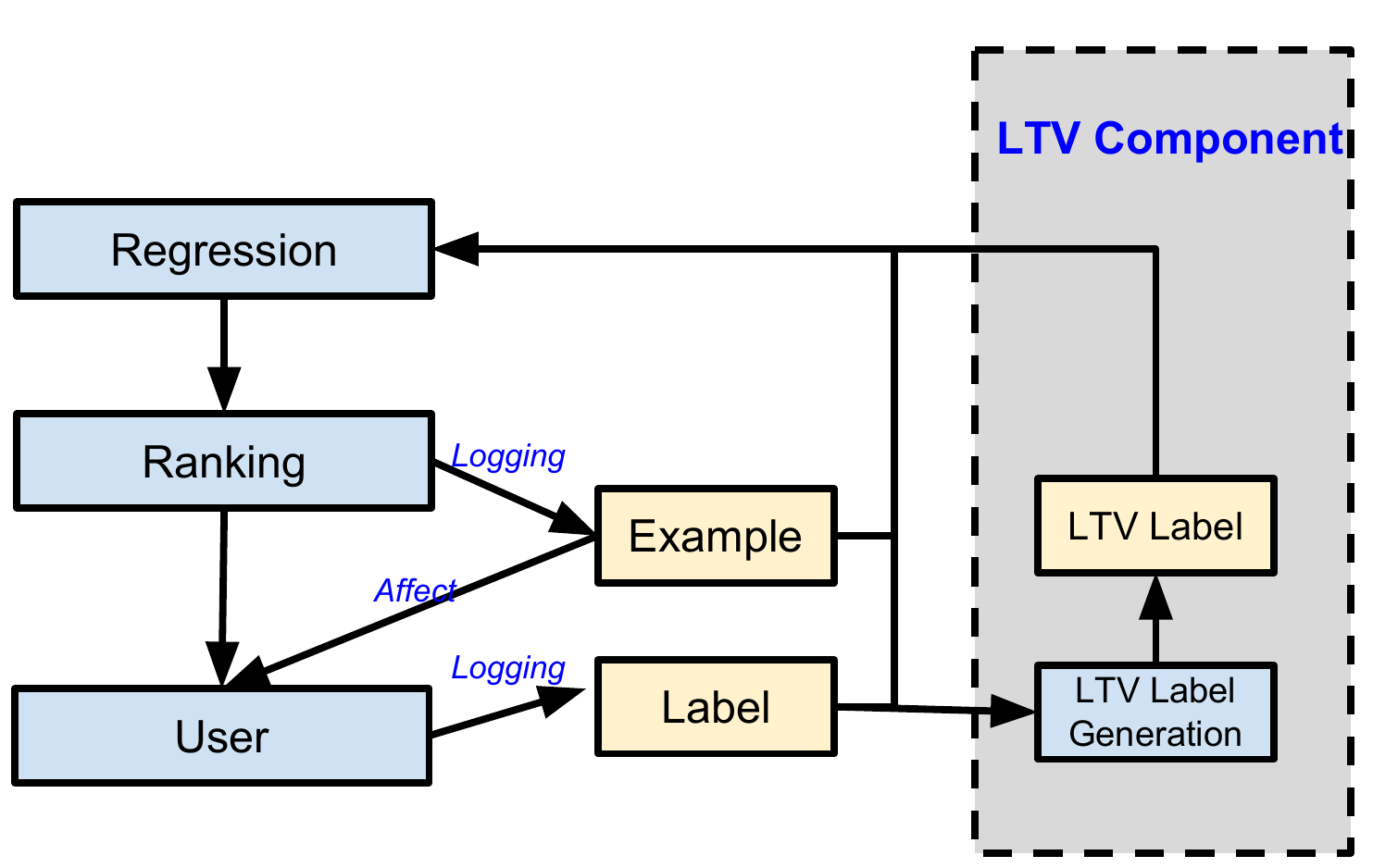}
    \vspace*{-2mm}
  \caption{Schematic View of a Non-myopic Recommender Training System}
  \label{fig:system_overview}
    \vspace*{-3mm}
\end{figure}

Models are trained periodically and pushed to the server. The ranker uses the latest model to recommend items and logs user feedback, which is used to train new models.
Using LTV labels, iterative model training and pushing can be viewed as a form of \emph{generalized policy iteration}~\cite{sutton:rlbook}.
Each trained DNN represents the value of the policy that generated the prior batch of training data, thus training is effectively \emph{policy evaluation}.
The ranker acts greedily with respect to this value function, thus performing \emph{policy improvement}.

LTV label generation is similar to DQN training \cite{mnih2015}. A main network learns the LTV of individual items, $\qbar^\pi(s,a)$---this network is easily extended from the existing myopic DNN. For stability, bootstrapped LTV labels (Q-values) are generated using a separate \emph{label network}.
We periodically copy the weights of the main network to the label network and use the (fixed) label network $\qbar_{label}(s, a)$ to compute LTV labels between copies.  LTV labels are generated using Eq.~(\ref{eq:slatesarsaupdate}).


\section{Empirical Evaluation: Live Experiments}
\label{sec:empiricalLive}

We tested the \SlateQ{} decomposition---specifically, the SARSA-TS algorithm, on YouTube (\texttt{https://www.youtube.com/}), a large-scale video recommender with $O(10^9)$ users and $O(10^8)$ items in its corpus.
The system is typical of many practical recommender systems with two main components. A \emph{candidate generator} retrieves a small subset (hundreds) of items from a large corpus that best match a user context.  The \emph{ranker} scores/ranks candidates using a DNN with both user context and item features as input.
It optimizes a combination of several objectives (e.g., clicks, expected engagement, several other factors).

The extant recommender system's policy is \emph{myopic}, scoring items for the slate using their \emph{immediate} (predicted) expected engagement. In our experiments, we replace the myopic engagement measure with an \emph{LTV estimate} in the ranker scoring function. We retain other predictions and incorporate them into candidate scoring as in the myopic model.
Our non-myopic recommender system maximizes \emph{cumulative} expected engagement, with user trajectories capped at $N$ days. Since homepage visits can be spaced arbitrarily in time, we use time-based rather than event-based discounting to handle credit assignment across large time gaps. 
If consecutive visits occur at times $t_1$ and $t_2$, respectively, the relative discount of the reward at $t_2$ is $\gamma^{(t_2 - t_1)/c}$, where $c$ is a parameter that controls the time scale for discounting.

Our model extends the myopic ranker using the practical methodology outlined in Section~\ref{sec:implementation}. Specifically, we learn a multi-task feedforward deep network \cite{zhang2017asurvey}, which learns $\qbar(s, i)$, the predicted long-term engagement of item $i$ (conditional on being clicked) in state $s$, as well as the immediate appeal $v(s, i)$ for pCTR/user choice computation (several other response predictions are learned, which are identical to those used by the myopic model).
The multi-task feedforward DNN network has 4 hidden layers of sizes 2048, 1024, 512, 256; and used ReLU activation functions on each of the hidden layers. Apart from the LTV/Q-value head, other heads include pCTR, and other user responses.
To validate our methodology, the DNN structure and all input features are identical to the production model which optimizes for short-term (myopic) immediate reward. The state is defined by user features (e.g., user’s past history, behavior and responses, plus static user attributes).
This also makes the comparison with the baseline fair.

The full training algorithm used in our live experiment is shown in Algorithm~\ref{algorithm:training}.
The model is trained using TensorFlow in a distributed training setup \cite{tensorflow2015-whitepaper} using stochastic gradient descent.
We train on-policy over pairs of consecutive start page visits, with LTV labels computed using Eq.~\ref{eq:slatesarsaupdate}, and use top-$k$ optimization for 
serving---i.e., we test SARSA-TS. The existing myopic recommender system (baseline) also builds slates greedily---i.e., MYOP-TS.

\begin{small}
\begin{algorithm}
  \caption{On-policy \SlateQ{} for Live Experiments}
  \begin{algorithmic}[1]
  \label{algorithm:training}
    \STATE \textbf{Parameters:}
    \vspace*{-3mm}
    \begin{itemize}\denselist
        \item $T$: the number of iterations.
        \item $M$: the interval to update label network.
        \item $\gamma$: discount rate.
        \item $\theta_{main}$: the parameter for the main neural network.
        \item $\qbar_{main}$: that predicts items' long-term value.
        \item $\theta_{label}$: the parameter for the label neural network $\qbar_{label}$.
        \item $\theta_{pctr}$: the parameter for the neural network that predicts items' pCTR.
    \end{itemize}
    \vspace*{-3mm}
    \STATE \textbf{Input:} $D_{training}=(s,A,C,L_{myopic}, s', A')$: the training data set.
    \vspace*{-3mm}
    \begin{itemize}\denselist
     \item $s$: current state features
     \item $A=(a_1,... ,a_k)$: recommended slate of items in current state; $a_{i}$ denotes item features
     \item $C=(c_1,...,c_k)$: $c_{i}$ denotes whether item $a_{i}$ is clicked
     \item $L_{myopic}=(l^{1}_{myopic},...,l^{k}_{myopic})$: myopic (immediate) labels
     \item $s'$: next state features
     \item $A'=(a_1',... ,a_k')$: recommended slate of items in next state.
    \end{itemize}
    \vspace*{-3mm}
    \STATE \textbf{Output:} Trained Q-network $\qbar_{main}$ that predicts items' long-term value.
    
    \STATE \textbf{Initialization} $\theta_{label}=0$, $\theta_{main}$ randomly, $\theta_{pctr}$   randomly
    \FOR{$i = 1 \dots T$}
      \IF{$i \bmod  M = 0$}
        \STATE$\theta_{label} \leftarrow \theta_{main}$
      \ENDIF
      
      \FOR{each example $(s,A,C,L_{myopic}, s', A') \in D_{training}$}
        \FOR{each item $a_i \in A$}
          \STATE update $\theta_{pctr}$ using click label $c_i$
          \IF{$a_i$ is clicked}
            \STATE probability: $\pctr(s', a_i', A')\leftarrow \pctr(s', a_i')/\sum_{a_i'\in A}pctr(s', a_i')$
            \STATE LTV label: $l^i_{ltv}\leftarrow l^i_{myopic} +\sum_{a_i'\in A'}\pctr(s', a_i', A') \qbar_{label}(s', a_i')$ 
            \STATE update $\theta_{main}$ using LTV label $l^i_{ltv}$
          \ENDIF
        \ENDFOR
      \ENDFOR
    \ENDFOR
  \end{algorithmic}
\end{algorithm}
\end{small}

We note that at serving time, we don't just choose the slate using the top-$k$ method, we also \emph{order} the slate presented to the user according to the item scores $v(s,i)\qbar^\pi(s,i)$ for each item $i$ (at state $s$). The reason for this is twofold. First, we expect that the user experience is positively impacted by placing more appealing items, that are likely to induce longer-term engagement, earlier in the slate. Second, the scrolling nature of the interface means that the slate size $k$ is not fixed at serving time---the number of inspected items varies per user-event (see discussion in Section~\ref{sec:cascadeOpt}).



We experimented with live traffic for three weeks, treating a small, but statistically significant, fraction of users to recommendations generated by our SARSA-TS LTV model. The control is a highly-optimized production machine learning model that 
optimizes for immediate engagement (MYOP-TS). Fig.~\ref{fig:twt_stats} shows the percentage increase in aggregate user engagement using LTV over the course of the experiment relative to the control, and indicates that our model outperformed the baseline on the key metric under consideration, consistently and significantly. Specifically, users presented recommendations by our model had sessions with greater engagement time relative to baseline.

\begin{figure}
\centering
   \includegraphics[width=0.70\linewidth]{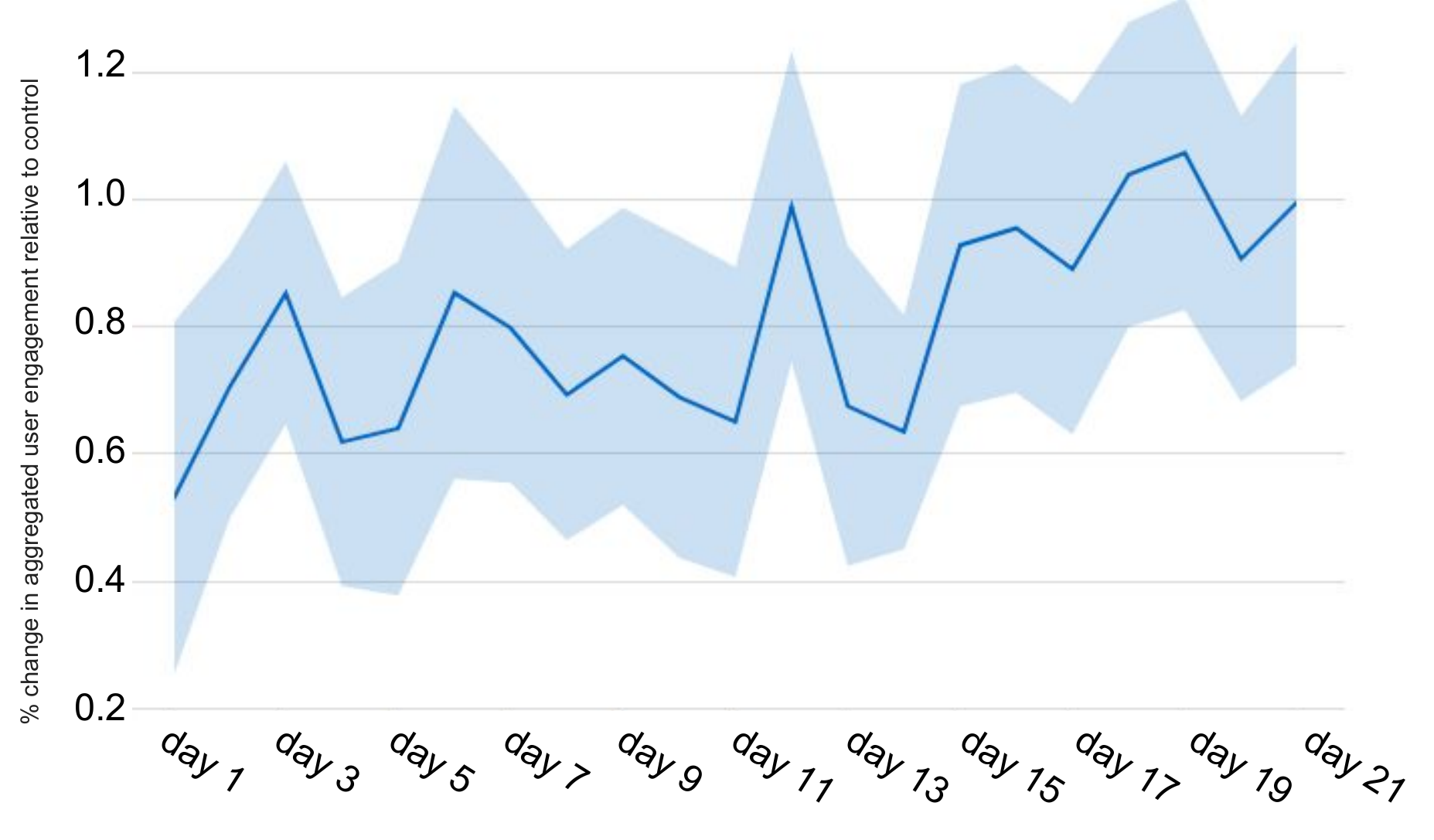}
     \vspace*{-3mm}
   \caption{Increase in user engagement over the baseline. Data points are statistically significant and within 95\% confidence intervals.
   \label{fig:twt_stats}}
     \vspace*{-3mm}
\end{figure}

Fig.~\ref{fig:engagement_by_position} shows the change in distribution of cumulative engagement originating from items at different positions in the slate. Recall that the number of items viewed in any user-event varies, i.e., experienced slates are of variable size and we show the first ten positions in the figure.
The results show that the users under treatment have more engaging sessions (larger LTVs) from items ranked higher in the slate compared to users in the control group, which suggests that top-$k$ slate optimization performs reasonably in this domain.\footnote{The apparent increase in expected engagement at position 10 is a statistical artifact due to the small number of events at that position: the number of observed events at each position decreases roughly exponentially, and position 10 has roughly two orders of magnitude fewer observed events than any of the first three positions.}

\begin{figure}
\centering
   \includegraphics[width=0.70\linewidth]{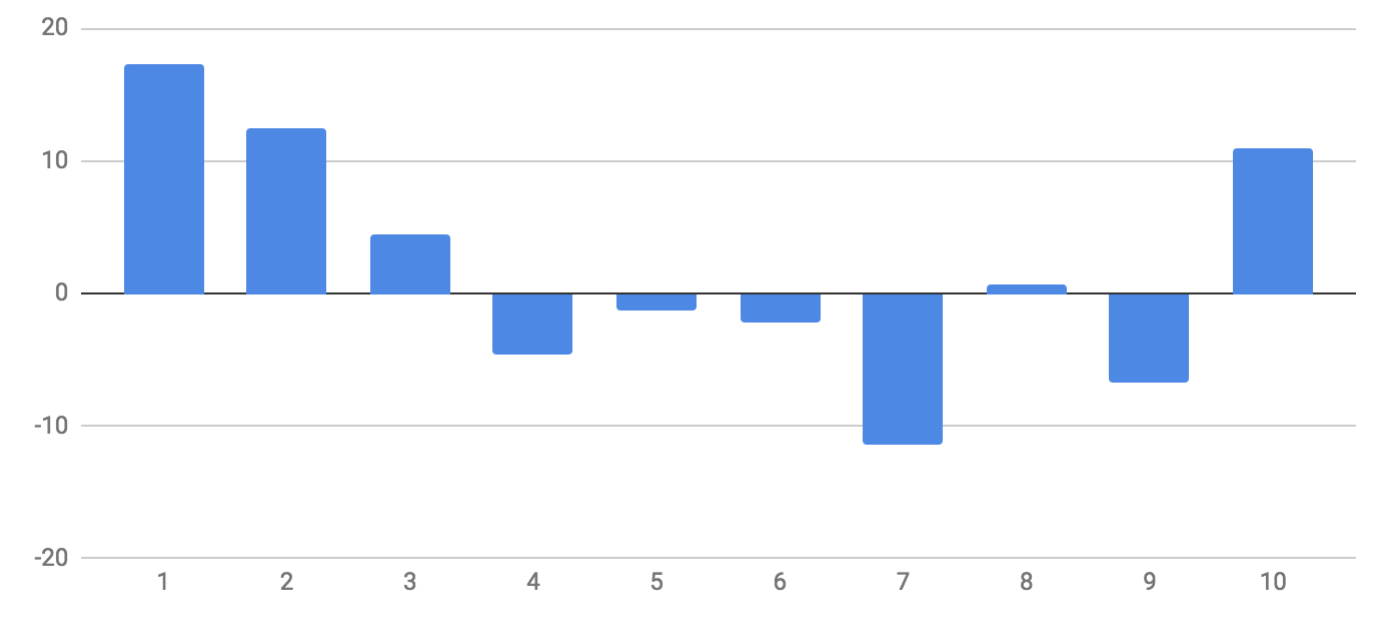}
     \vspace*{-3mm}
   \caption{Percentage change in long-term user engagement vs.\ control ($y$-axis) across positions in the slate ($x$-axis). Top 3 positions account for approximately 95\% of engagement.}
   \label{fig:engagement_by_position}
     \vspace*{-3mm}
\end{figure}

\section{Conclusion}

In this work, we addressed the problem of optimizing long-term user engagement in slate-based recommender systems using reinforcement learning. Two key impediments to the use of RL in large-scale, practical recommenders are (a) handling the combinatorics of slate-based action spaces; and (b) constructing the underlying representations.

To handle the first, we developed \SlateQ{}, a novel decomposition technique for slate-based RL that allows for effective TD and Q-learning using LTV estimates for individual items. It requires relatively innocuous assumptions about user choice behavior and system dynamics, appropriate for many recommender settings. The decomposition allows for effective TD and Q-learning by reducing the complexity of generalization and exploration to that of learning for individual items---a problem routinely addressed by practical myopic recommenders. Moreover, for certain important classes of choice models, including the conditional logit, the slate optimization problem can be solved tractably using optimal LP-based and heuristic greedy and top-$k$ methods. Our results show that \SlateQ{} is relatively robust in simulation, and can scale to large-scale commercial recommender systems like YouTube.

Our second contribution was a practical methodology for the introduction of RL to extant, myopic recommenders. We proposed the use of existing myopic models to bootstrap the development of Q-function-based RL methods, in a way that allows the substantial reuse of current training and serving infrastructure. Our live experiment in YouTube recommendation exemplified the utility of this methodology and the scalability of \SlateQ{}. It also demonstrated that using LTV estimation can improve user engagement significantly in practice.

There are a variety of future research directions that can extend the work here. First, our methodology can be extended by relaxing some of the assumptions we made regarding the interaction between user choice and system dynamics. For instance, we are interested in models that allow unconsumed items on the slate to influence user latent state and choice models that allow for multiple items on a slate to be used/clicked. Further analysis of, and the development of corresponding optimization procedures for, additional choice models using \SlateQ{} remains of intense interest (e.g., hierarchical model such as nest logit). In a related vein, methods for simultaneous learning of choice models, or their parameters, while learning Q-values would be of great practical value. Finally, the simulation environment has the potential to serve as a  platform for additional research on the application of RL to recommender systems. We hope to release a version of it to the research community in the near future.



\paragraph{Acknowledgments.} Thanks to Larry Lansing for system optimization and the IJCAI-2019 reviewers for helpful feedback.



\end{document}